%% file: main.tex
\documentclass[letterpaper, 10 pt, conference]{ieeeconf} 
\IEEEoverridecommandlockouts                             
\usepackage{graphicx}
\usepackage{amsmath}
\usepackage{amssymb}
\usepackage[linesnumbered,ruled]{algorithm2e}
\usepackage{color}
\usepackage{url}

\usepackage{amsthm}
\usepackage{subfigure}

\newtheorem{theorem}{Theorem}
\newtheorem{lemma}{Lemma}

\newcommand{\WinnerTakesAll}{\textsc{\texttt{WinnerTakesAll~}}}
\newcommand{\Bottleneck}{\textsc{\texttt{Bottleneck~}}}

\newcommand{\argmax}{\arg\!\max}

\title{\LARGE \bf
Distributed Simultaneous Action and Target Assignment for Multi-Robot Multi-Target Tracking 
}

\author{Yoonchang Sung, Ashish Kumar Budhiraja, Ryan K. Williams and Pratap Tokekar%
\thanks{The authors are with the Department of Electrical and Computer Engineering, Virginia Tech, USA. \texttt{\small \{yooncs8, ashishkb, rywilli1, tokekar\}@vt.edu}.}%
\thanks{This material is based upon work supported by the National Science Foundation under Grant No. 1637915.}
\thanks{The authors would like to thank Dr. Jukka Suomela from Aalto University for fruitful discussion.}
}

\begin{document}

\maketitle
\thispagestyle{empty}
\pagestyle{empty}

\begin{abstract}

We study two multi-robot assignment problems for multi-target tracking. We consider distributed approaches in order to deal with limited sensing and communication ranges. We seek to simultaneously assign trajectories and targets to the robots. Our focus is on \emph{local} algorithms that achieve performance close to the optimal algorithms with limited communication. We show how to use a local algorithm that guarantees a bounded approximate solution within $\mathcal{O}(h\log{1/\epsilon})$ communication rounds. We compare with a greedy approach that achieves a $2$--approximation in as many rounds as the number of robots. Simulation results show that the local algorithm is an effective solution to the assignment problem.

\end{abstract}

\section{Introduction} \label{sec:int}

We study the problem of assigning robots with limited Field-of-View (FoV) sensors to track multiple moving targets. We focus on scenarios where the number of robots is large and solving the problem locally rather than centrally is desirable. 
The robots may have a limited communication range and bandwidth. 
As such, we seek assignment algorithms that rely on local information and limited, local communication with the neighboring robots. 
We assume that each robot has a number of motion primitives to choose from. The assignment of targets to track is therefore coupled with the selection of motion primitives for each robot. We term this as the distributed Simultaneous Action and Target Assignment (SATA) problem.

A motion primitive is a local trajectory obtained by applying a sequence of actions. We interchangeably use motion primitives to refer to the trajectories as well as the final state on them. 
A motion primitive can track a target if the target is in the FoV of the robot. 
The set of targets tracked by different motion primitives may be different (Figure~\ref{fig:description}). 
Our goal is to assign motion primitives to the robots so as to track the most number of targets. 
This problem can be viewed as a version of set cover~\cite{suomela2013survey} where every target must be covered by at least one motion primitive. However, we have the additional constraint that only one motion primitive can be chosen per robot. This is called as a packing problem~\cite{suomela2013survey}. The combination of these two problems is called a Mixed Packing and Covering Problem (MPCP)~\cite{young2001sequential}. 

\begin{figure}[thpb]
\centering
\includegraphics[width=0.75\columnwidth]{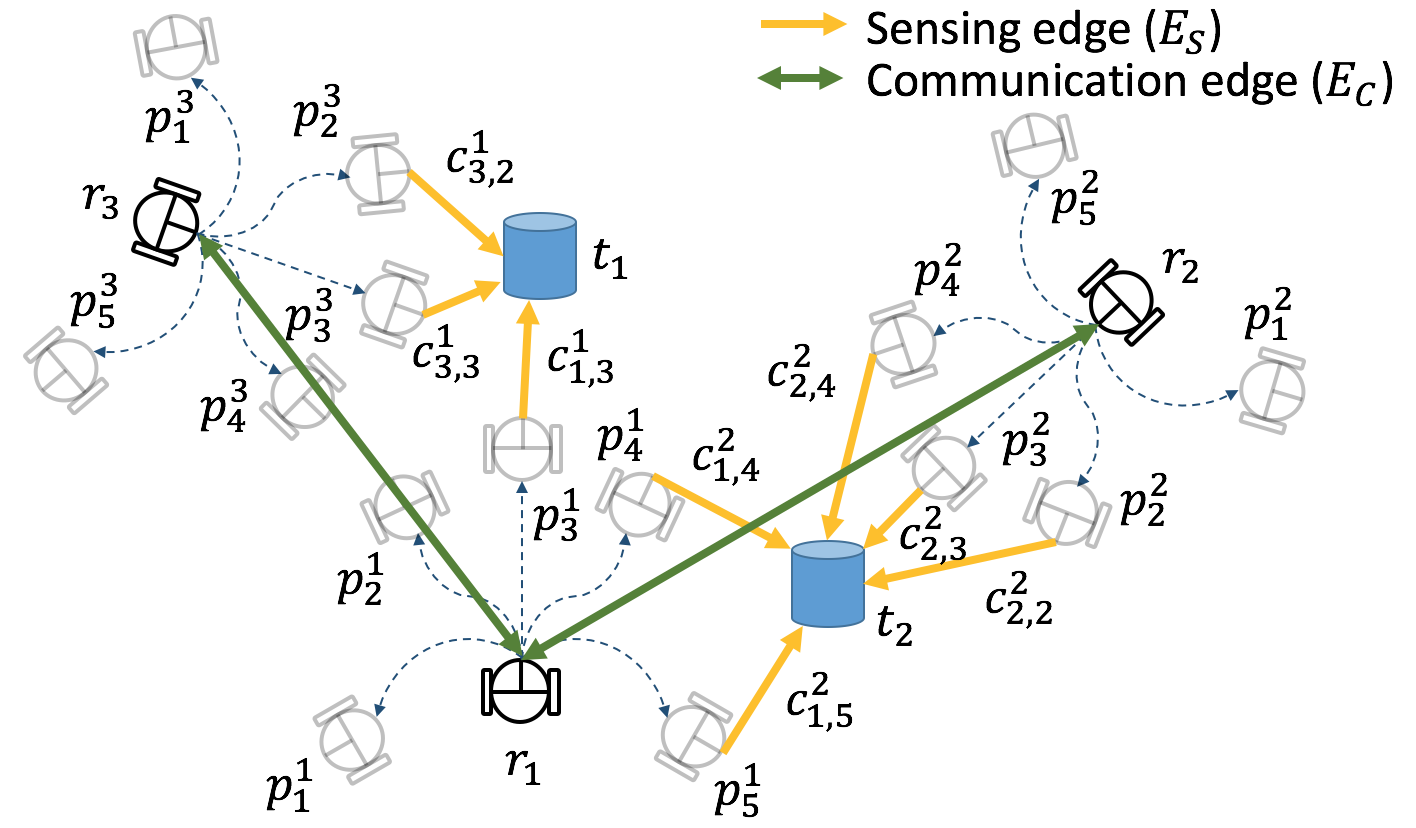}
\caption{Description of multi-robot task allocation for multi-target tracking. There are five motion primitives per robot.}
\label{fig:description}
\end{figure}

The problem can also be formulated as a (sub)modular maximization problem subject to a partition matroid constraint~\cite{nemhauser1978analysis}. A sequential greedy algorithm, where the robots take turns to greedily choose motion primitives, is known to yield a $2$--approximation for this problem~\cite{tokekar2014multi}. The sequential greedy algorithm requires at least as many communication rounds as the number of robots, which may be too slow in practice. Consequently, we focus on \emph{local} algorithms.

\begin{figure}[thpb]
\centering
\includegraphics[width=0.45\columnwidth]{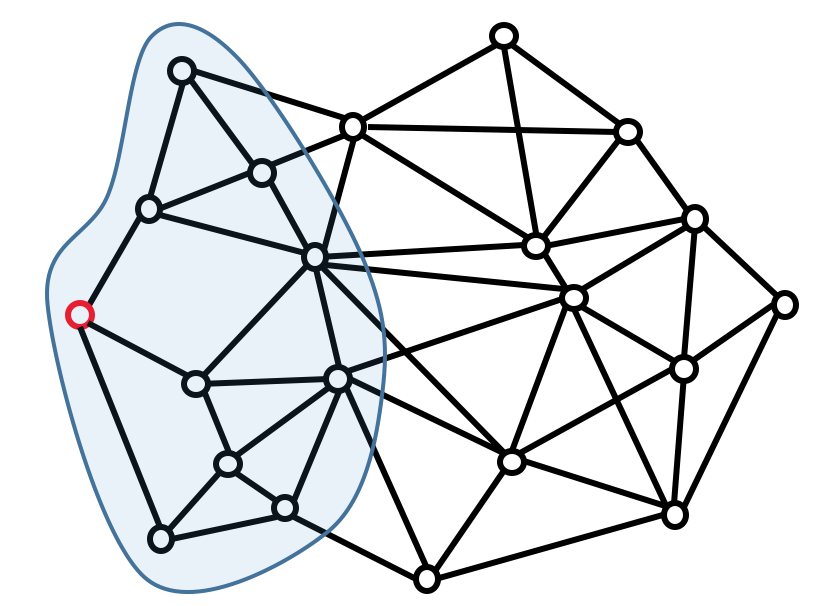}
\caption{Communication graph. The blue region indicates a radius--2 neighborhood of the red node. The red node can be unaware of the entire network topology. A local algorithm that works for the red node only requires a local information of nodes in the blue region. The same local algorithm is applied to all nodes in the network.}
\label{fig:local}
\end{figure}

A local algorithm~\cite{suomela2013survey} is a constant-time distributed algorithm that is independent of the size of a network. This enables a robot to depend only on the local inputs in a fixed-radius neighborhood (Figure~\ref{fig:local}). The robot does not need to know information beyond its local neighborhood, thereby achieving better scalability. Flor{\'e}en et al.~\cite{floreen2011local} proposed a local algorithm to solve MPCP using max-min/min-max Linear Programming (LPs) in a distributed manner. We show how to leverage this algorithm to solve SATA.

There have been many studies~\cite{khan2016cooperative,robin2016multi} on cooperative target tracking in both control and robotics communities. 
Yu et al.~\cite{yu2015cooperative} presented an auction-based decentralized algorithm for cooperative tracking of a mobile target. Capitan et al.~\cite{capitan2013decentralized} proposed a decentralized cooperative multi-robot algorithm using auctioned partially observable Markov decision processes. The performance of decentralized data fusion was successfully shown under limited communication but theoretical bounds on communication rounds were not covered. 

Morbidi et al.~\cite{morbidi2013active} presented a gradient-based control scheme for active multi-target tracking. Their focus was not on distributed control policies. Ahmad et al.~\cite{ahmad2013cooperative} proposed a least squares minimization technique for cooperative multi-target tracking. However, they focused on localization, not on the multi-robot multi-target assignment. Pimenta et al.~\cite{pimenta2009simultaneous} adopted Voronoi partitioning to develop a distributed multi-target tracking algorithm. However, their interest lied in covering an environment coupled with multi-target tracking.

The works in~\cite{kulatunga2006ant},~\cite{bellingham2003multi} and~\cite{eun2009cooperative} proposed algorithms to solve simultaneous task allocation and path planning, similar to SATA. However, their approaches are centralized. 
In our prior work~\cite{tokekar2014multi} we addressed the problem of selecting trajectories for robots that can track the maximum number of targets using a team of robots. No bound on the number of communication rounds was given, possibly resulting in all-to-all communication in the worst case. Instead, we explicitly bound the amount of communication.

Our contributions are as follows: (1) We show how to adapt the local algorithm for solving SATA. (2) We perform empirical comparisons with greedy 
and baseline centralized algorithms. (3) We demonstrate the applicability of the proposed algorithm through simulations.



\section{Problem Description} \label{sec:prob}
Let $R$ and $T$ be sets of robots and targets. $R(k)=\{\textbf{r}_{1}(k),...,\textbf{r}_{i}(k),...,\textbf{r}_{|R|}(k)\}$ denotes the state of robots at time $k$ and $T(k+1)=\{\textbf{t}_{1}(k+1),...,\textbf{t}_{j}(k+1),...,\textbf{t}_{|T|}(k+1)\}$ denotes the predicted state of targets at time $k+1$. We assume that the targets can be uniquely detected and two robots know if they are observing the same target. 
Motion primitives of $i$-th robot $\textbf{r}_{i}(k)$ at time $k$ are denoted by $P^i(k)=\{\textbf{p}_{1}^i(k),...,\textbf{p}_{m}^i(k),...,\textbf{p}_{{|P^i|}}^i(k)\}$. Note again that the term \emph{motion primitives} in this paper represents the future state of a robot after the corresponding feasible control input is applied starting at time $k$. 

We denote the sensing and communication ranges by $\mathcal{RS}$ and $\mathcal{RC}$. Predicted $j$-th target $\textbf{t}_{j}(k+1)$ at time $k+1$ is said to be observable from $m$-th motion primitive of robot $i$ iff $\textbf{t}_{j}(k+1)\in \mathcal{RS}(\textbf{p}_{m}^i(k))$. Likewise, $\alpha$-th robot can communicate with $\beta$-th robot iff $\textbf{r}_{\alpha}(k)\in \mathcal{RC}(\textbf{r}_{\beta}(k))$ and $\textbf{r}_{\beta}(k)\in \mathcal{RC}(\textbf{r}_{\alpha}(k))$. 
We assume that $\mathcal{RC}(\cdot)> 2\mathcal{RS}(\cdot)$. This implies if there is a target $j$ such that $\textbf{t}_{j}(k+1)\in \mathcal{RS}(\textbf{p}_{m}^\alpha(k))$ and $\textbf{t}_{j}(k+1)\in \mathcal{RS}(\textbf{p}_{m}^\beta(k))$, then $\alpha$-th and $\beta$-th robots can communicate with each other. Therefore, neighboring robots can share their local information with each other when they observe the same targets. 

We also assume that all the robots have synchronous clocks leading to synchronous rounds of communication. This is required in order to employ a greedy algorithm and local algorithm that will be covered in Section~\ref{sec:distributed}.


Each robot must choose one of its motion primitives to maximize the tracking objective. We first show how to formulate this as an Integer Linear Program (ILP). We define two binary variables: $x_{m}^{i}$  and $y_{i}^{j}$. $x_{m}^{i}=1$ if $\textbf{p}_m^i$ is selected by $\textbf{r}_i$ and $0$ otherwise. $y_{i}^{j}=1$ if $\textbf{r}_i$ is assigned to $\textbf{t}_j$ and $0$ otherwise. It follows:
\begin{equation}
\sum_{\textbf{p}_m^i\in P^i}{x_{m}^{i}}\leq{1}\ \ \forall \textbf{r}_i\in R,\quad
\sum_{\textbf{r}_i\in R}{y_{i}^{j}}\leq{1}\ \ \forall \textbf{t}_j\in T.
\label{eqn:variable}
\end{equation}
%
%
The objective is to assign the robots/primitives such that all targets are \emph{equitably} covered:
\begin{equation}
\argmax_{x_{m}^{i}}\quad\min_{\textbf{t}_j\in T}\sum_{\textbf{r}_i\in R}\sum_{\textbf{p}_m^i\in P^i}c_{i,m}^jx_{m}^{i},
\label{eqn:bottleneck}
\end{equation}
where $c_{i,m}^j$ denotes weights on sensing edges $E_S$ between $m$-th motion primitive of $i$-th robot and $j$-th target. $c_{i,m}^j$ can represent, for example, the distance between $t_j$ and $p_m^i$. Note that in case $t_j$ and $p_m^i$ have uncertainty associated with them, we can use the Bhattacharyya distance between the corresponding distributions to compute $c_{i,m}^j$. 

Consequently, an optimal motion primitive $\textbf{p}_m^{i*}$ for all robots can be selected based on $x_{m}^{i}$ and $y_{i}^{j}$. We term this as the \Bottleneck version of SATA.

We also define a \WinnerTakesAll variant of SATA where the objective is given by,
\begin{equation}
\argmax_{x_{m}^{i},y_{i}^{j}}\sum_{\textbf{t}_j\in T}\left(\sum_{\textbf{r}_i\in R}y_{i}^{j}\left(\sum_{\textbf{p}_m^i\in P^i}c_{i,m}^jx_{m}^{i}\right)\right).
\label{eqn:objective}
\end{equation} 
Here the goal is  to maximize the quality of tracking (alternatively, number of targets that are tracked).





If we fix $c_{i,m}^j=1$ when $p_m^i$ can observe $t_j$ and zero otherwise, then the objective function becomes equal to the number of targets tracked. 

Both versions of the SATA problem are NP-Hard~\cite{vazirani2001approximation}. The \WinnerTakesAll version can be optimally solved using a Quadratic Mixed Integer Linear Programming (QMILP) solver in the centralized setting. Our main contributions are to show how to solve both problems in a distributed manner: an LP-relaxation of the \Bottleneck variant using a local algorithm; and the \WinnerTakesAll variant using a greedy algorithm. The following theorems summarize the main contributions of our work.

\begin{theorem}~\label{theorem:proposed_local}
Let $\bigtriangleup_R\geq2$ be the maximum number of motion primitives per robot and $\bigtriangleup_T\geq2$ be the maximum number of motion primitives that can see a target. There exists a local algorithm that finds an $\bigtriangleup_R(1+\epsilon)(1+1/h)(1-1/\bigtriangleup_T)$ approximation in $\mathcal{O}(h\log{1/\epsilon})$ synchronous communication rounds for the LP-relaxation of the \Bottleneck version of SATA problem, where $h$ and $\epsilon>0$ are parameters.
\end{theorem}
The proof follows directly from the existence of the local algorithm described in the next section. If $\bigtriangleup_R=1$ or $\bigtriangleup_T=1$, there exist local algorithms that give the optimal solution (c.f. Theorem 1 from~\cite{floreen2011local}).

\begin{theorem}~\label{theorem:proposed_greedy}
There exists a $2$--approximation greedy algorithm for the \WinnerTakesAll version of the SATA problem for any $\epsilon>0$ in polynomial time.
\end{theorem}
This  follows from the fact that this is a modular maximization problem subject to a partition matroid constraint~\cite{nemhauser1978analysis}.

\section{Distributed Algorithms} \label{sec:distributed}

\subsection{Local Algorithm} \label{subsec:local}


In this section, we show how to solve the $\textsc{\texttt{Bottleneck~}}$ version of the SATA problem using a local algorithm. We adapt the local algorithm for solving max-min LPs~\cite{floreen2011local} to solve the SATA problem in a distributed manner. 

Consider the tripartite, weighted, and undirected graph, $\mathcal{G}=(R\cup{P}\cup{T},E)$ shown in Figure~\ref{fig:general}. Each edge $e\in{E}$ is either $e=\{\textbf{r}_i,\textbf{p}_m^i\}$ with weight 1 or $e=\{\textbf{t}_j,\textbf{p}_m^i\}$ with weight $c_{i,m}^{j}\in{C}$. The maximum degree among robot nodes $\textbf{r}_i\in{R}$ is denoted by $\bigtriangleup_R$ and among target nodes $\textbf{t}_j\in{T}$ is $\bigtriangleup_T$. Each motion primitive $\textbf{p}_m^i\in{P}$ is associated with a variable $x_{m}^{i}$.
The upper part of $\mathcal{G}$ in Figure~\ref{fig:general} is related with a packing problem (Equation~\ref{eqn:objective}). The lower part is related with the covering problem. The \Bottleneck version (Equation~\ref{eqn:bottleneck}) can be rewritten as a linear relaxation of ILP:
\begin{equation}
\begin{split}
\mbox{maximize}\ \ \ &w \\
\mbox{subject to}\ \ \ &\sum_{\textbf{p}_m^i\in{P^i}}x_{m}^{i}\leq{1}\ \ \forall{\textbf{r}_i\in{R}} \\
&\sum_{\textbf{r}_i\in{R}}\sum_{\textbf{p}_m^i\in{P^i}}c_{i,m}^{j}x_{m}^{i}\geq{w}\ \ \forall{\textbf{t}_j\in{T}} \\
&\ \ \ \ \ \ \ \ \ \ x_{m}^{i}\geq{0}\ \ \forall{\textbf{p}_m^i\in{P^i}}.
\end{split}
\label{eqn:mpcp}
\end{equation}

\begin{figure}[thpb]
\centering
\includegraphics[width=0.80\columnwidth]{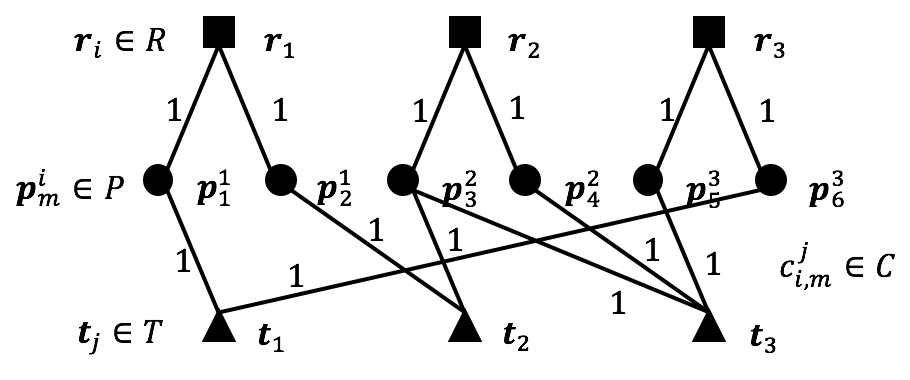}
\caption{One instance of a graph for MPCP when there are three robot nodes, six motion primitive nodes and three target nodes. The objective in this example is to maximize the number of targets being tracked. Hence, all weights are set to $1$. In general, weights can be arbitrary values.}
\label{fig:general}
\end{figure}

Floreen et al.~\cite{floreen2011local} presented a local algorithm to solve MPCP in Equation~\ref{eqn:mpcp} in a distributed fashion. 
We show how to adapt this to the \Bottleneck version of SATA.
An overview of our algorithm is given in Figure~\ref{fig:flowchart}. We describe the main steps in the following.

\begin{figure}[thpb]
\centering
\includegraphics[width=0.8\columnwidth]{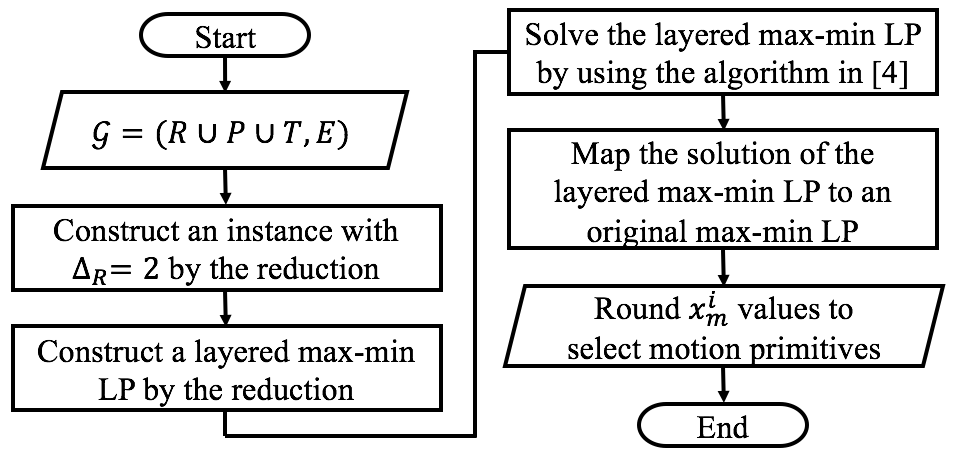}
\caption{Flowchart of the proposed local algorithm.}
\label{fig:flowchart}
\end{figure}

\subsubsection{Local Algorithm from~\cite{floreen2011local}}
The local algorithm in~\cite{floreen2011local} requires~$\bigtriangleup_R=2$. However, they also present a simple local technique to split nodes in the original graph with $\bigtriangleup_R > 2$ into multiple nodes making $\bigtriangleup_R = 2$. Then, a \emph{layered} max-min LP  is constructed with $h$ layers (Figure~\ref{fig:layered}). The details of the construction of the layered graph is given in Section 4 of~\cite{floreen2011local}. $h$ is a user-defined parameter that  trades-off computational time with optimality. Layered graph breaks the symmetry that inherently exists in an original graph. 

\begin{figure}[thpb]
\centering
\includegraphics[width=0.55\columnwidth]{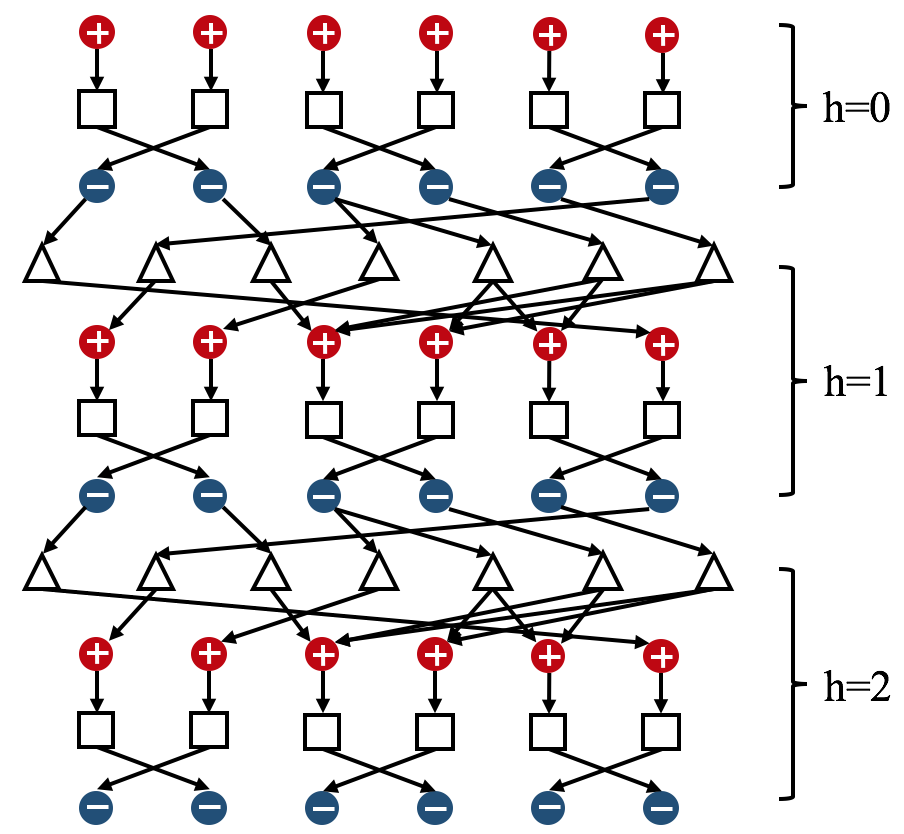}
\caption{Graph of the layered max-min LP with $h=2$ that is obtained from the original graph of Figure~\ref{fig:general} after applying the local algorithm.}
\label{fig:layered}
\end{figure}

Authors in~\cite{floreen2011local} proposed a recursive algorithm to compute a solution of the layered max-min LP. The solution for the original max-min LP can be obtained by mapping from the solution of the layered one. The obtained solution corresponds to values of $x_m^i$. They proved that the resulting algorithm gives a constant-factor approximation ratio.
\begin{theorem}
There exist local approximation algorithms for max-min and min-max LPs with the approximation ratio $\bigtriangleup_R(1+\epsilon)(1+1/h)(1-1/\bigtriangleup_T)$ for any $\bigtriangleup_R\geq2$, $\bigtriangleup_T\geq2$, and $\epsilon>0$, where $h$ denotes the number of layers.
\end{theorem}
\begin{proof}
Please refer to Corollary 4.7 from~\cite{floreen2011local} for a proof.
\end{proof}
Each node in the layered graph carries out its local computation. Each node also receives and sends information from and to neighbors at each synchronous communication round. The layered graph is constructed in a local fashion without requiring any single robot to know the entire graph.

\subsubsection{Realization of Local Algorithm for SATA} \label{subsubsec:realize}

To apply the local algorithm to distributed SATA, each node and edge in a layered graph must be realized at each time step. In our case, the only computational units are the robots. Nodes that correspond to motion primitives, $\textbf{p}_m^i\in{P}$, can be realized by the corresponding robot $\textbf{r}_i\in{R}$. Nodes corresponding to the targets are also realized by the robots. A target $\textbf{t}_j$ is realized by a robot $\textbf{r}_i$ satisfying $\textbf{t}_{j}\in \mathcal{RS}(\textbf{p}_{m}^i)$. If there are multiple robots whose motion primitives can sense the target, they can arbitrarily decide who realizes the target node in a constant number of communication rounds. 


After applying the local algorithm, each robot obtains $x_{m}^{i}$ for corresponding $\textbf{p}_m^i$. However, due to the LP relaxation, $x_{m}^{i}$ will not necessarily be binary. For each robot we set the highest $x_{m}^{i}$ equal to one and all others as zero. We shortly show that the resulting solution after rounding is still close to optimal in practice. Furthermore, increasing $h$ results in better solutions at the expense of more communication. $h=0$ is equivalent to the greedy approach where no robots communicate with each other. 


The following table shows the result of applying the local algorithm to the graph in Figure~\ref{fig:general}. Three different values for $h$ were tested: $2$, $10$, and $30$. In all cases, $\textbf{p}_3^2$ and $\textbf{p}_6^3$ have larger values of $x_p$ than other nodes. Thus, the robot $\textbf{r}_2$ and the robot $\textbf{r}_3$ will select $\textbf{p}_3^2$ and $\textbf{p}_6^3$ as motion primitives, respectively after employing a rounding technique to $x_p$'s. 

As the number of layers increases, the more distinct the $x_p^i$ values returned by the algorithm. Another interesting observation is that robot $\textbf{r}_1$ has the same equal value on both motion primitives of its own no matter how many number of layers is used. This is because all the targets are already observed by robots $\textbf{r}_2$ and $\textbf{r}_3$ with higher values. 

\begin{table}[h]
\centering
\begin{center}
\begin{tabular}{ c c c c c } 
 \hline
 \hline \\[-1em]
 $\textbf{p}_m^i $ & $x_{m}^{i}$ & $h = 2$\ & $h = 10$ &
$h = 30$ \\ \\[-1em]
 \hline
 \hline \\[-1em]
 $\textbf{p}_1^1$ & $x_1^1 = $ & 0.5000 & 0.5000 & 0.5000 \\ \\ \\[-2em]

 $\textbf{p}_2^1$ & $x_2^1 = $ & 0.5000 & 0.5000 & 0.5000 \\ \\[-1em]
 \hline \\[-1em]
 $\textbf{p}_3^2$ & $x_3^2 = $ & 0.6667 & 0.7591 & 0.7855 \\ \\ \\[-2em]

 $\textbf{p}_4^2$ & $x_4^2 = $ & 0.3333 & 0.2409 & 0.2145 \\ \\[-1em]
 \hline \\[-1em]
 $\textbf{p}_5^3$ & $x_5^3 = $ & 0.3333 & 0.2409 & 0.2145 \\ \\ \\[-2em]

 $\textbf{p}_6^3$ & $x_6^3 = $ & 0.6667 & 0.7591 & 0.7855 \\ \\[-1em]
 \hline
 \hline
\end{tabular}
\end{center}
\caption{Solution returned by the local algorithm for the example shown in Figure~\ref{fig:general} with the varying number of layers, $h$. }
\label{table:1}
\end{table}

Algorithm~\ref{alg:local} explains the overall scheme of each robot for a distributed SATA. We solve the SATA problem at each time step. In principle, we can replace each motion primitive with a longer horizon trajectory and plan for multiple time steps without affecting the computation time significantly.


\begin{algorithm} \label{alg:local}
    \SetKwInOut{Input}{Input}
    \SetKwInOut{Output}{Output}
    \For{$\textbf{r}_{i,k}\in R_k$}
    {
      $\textbf{p}_{m,k}^i\in P_k^i\leftarrow$ComputeMotionPrimitives($\textbf{r}_{i,k}$)
      
      Find targets that can be sensed by $\textbf{p}_{m,k}^i$
      
      Construct a $h$-hop communication graph
      
      Apply local algorithm
      
      $\hat{x}_{m}^{i}\leftarrow$ Rounding\big($x_{m}^{i}$\big)
      
      $\textbf{p}_m^{i*}\leftarrow$ Motion primitive with $\hat{x}_{m}^{i} = 1$
      
      ApplyAction\big($\textbf{p}_m^{i*}$\big)
      
      $k\leftarrow k+1$
      
    }
    \caption{Local algorithm}
\end{algorithm}

One of the advantages of the local algorithm is that even if the communication graph is disconnected, each component of the graph can run the local algorithm independently without affecting the solution quality. 
The algorithm also allows for the number of robots and targets to change. Since each robot determines its neighbors at each time step, any new robots or targets will be identified and become part of the local layered graphs at the next planning timestep. 

\subsection{Greedy Algorithm} \label{subsec:greedy}

We use the greedy algorithm proposed in~\cite{tokekar2014multi} and~\cite{dames2015detecting} as the baseline comparison. The greedy algorithm requires a specific ordering of the robots given in advance. The first robot greedily chooses a motion primitive that can maximize the number of targets being observed. Those observed targets are removed from the consideration. Then, the second robot makes its choice; this repeats for the rest of robots. Note again that the greedy algorithm is for the \WinnerTakesAll version of SATA.

\begin{algorithm} \label{alg:greedy}
    \SetKwInOut{Input}{Input}
    \SetKwInOut{Output}{Output}
    \Input{Order of robots $R$}
    
    Initialize $w(\textbf{t}_j)=0\ \forall \textbf{t}_j\in T$

\For{$\textbf{r}_{i}\in R$}
{
\For{$\textbf{p}_{m}^i\in P^i$}
{
Compute $c_{i,m}^j$
$w^\prime(\textbf{p}_{m}^i)=\sum_{\textbf{t}_{j}}\max\{w(\textbf{t}_{j}),c_{i,m}^j\}$
}

Determine $x_m^i=\argmax w^\prime(\textbf{p}_m^i)\ \forall \textbf{p}_m^i\in P^i$

Update $w(\textbf{t}_j)=\max\{w(\textbf{t}_j),c_{i,m}^j\}\ \forall \textbf{t}_j\in T$

}
$y_i^j\leftarrow0\ \forall \textbf{r}_{i}\in R,\ \textbf{t}_j\in T$

\For{$\textbf{t}_j\in T$}
{
$\textbf{r}_i^*\leftarrow\argmax_{\textbf{r}_i\in R}\sum_{\textbf{p}_m^i}c_{i,m}^jx_m^i$

$y_{i^*}^j\leftarrow 1$
}
    \caption{Greedy algorithm}
\end{algorithm}

As shown in Algorithm~\ref{alg:greedy}, the greedy algorithm runs in $|R|$ communication rounds at each time step. We define two functions: $w(\textbf{t}_j)$ gives a quality of tracking for $j$-th target; and $w^\prime(\textbf{p}_m^i)$ gives the sum of quality of tracking over all feasible targets using $m$-th motion primitive of $i$-th robot. If, for example, $c_{i,m}^j$ is used as a distance metric, the $\max$ ensures that the quality of tracking for $j$-th target is only given by the distance of the nearest robot/primitive. That is, even if multiple primitives can track the same $\textbf{t}_j$, when counting the quality we only care about the closest one. The total quality will then be the sum of qualities for each target.

\begin{lemma}~\label{lemma:greedy}
Greedy algorithm of Algorithm~\ref{alg:greedy} gives a feasible solution for the \WinnerTakesAll version of SATA.
\end{lemma}
The proof is given in the appendix.

A centralized-equivalent approach is one where the robots all broadcast their local information until some robot has received information from all others. This robot can obtain a centralized solution. A centralized-equivalent approach for a complete communication runs in 2 communication rounds for receiving and sending data to neighbors. However, the local algorithm and greedy algorithm take $h\log(1/\epsilon)$ and $|R|$ communication rounds, respectively. Note that $h\ll|R|$ for most practical cases. 

\section{Simulations} \label{sec:sim}
In this section, we evaluate empirically the performance of the local algorithm and greedy algorithm in two settings. 




\subsection{Comparison Study}

We compare the proposed algorithms with the QMILP solution. The greedy algorithm and QMILP solve the \WinnerTakesAll problem and local algorithm solves the \Bottleneck problem. However, we compare the total number of targets covered by both approaches.
We used TOMLAB~\cite{QMILP} to get the QMILP solution. TOMLAB works with MATLAB and uses IBM's CPLEX optimizer. An Intel® Core™ i7-5500U CPU @ 2.40GHz~x~4 laptop with 16 GB memory took a maximum time of around 4 seconds to solve an instance with 200 targets and average target degree of 2. Most of instances were solved in less than 2 seconds.

We randomly generated graphs similar to Figure~\ref{fig:general} with a given average degree for comparison. We start with the upper half of the graph, connecting each robot to its two motion primitives. Then we iterate through each motion primitive and randomly choose a target node to create an edge. Next, we iterate through target nodes and randomly choose a motion primitive to create an edge. We also add random edges to connect disconnected components (to keep the implementation simpler). We repeat this in order to get the required graph. We create new edges to random primitives till we achieve the desired degree. We generated cases by varying the degree of targets, number of targets, and number of robots using the method described above.

Figure~\ref{fig:cmp_plot} shows the minimum, maximum, and the mean number of targets covered by the local algorithm, greedy algorithm and QMILP running 100 random instances for every setting of the parameters. We also show the number of targets covered when choosing motion primitives randomly as a baseline. We observe that the local algorithm with $h=2$ performs comparatively to the optimal algorithm, and is always better than the baseline. In all the figures, $\Delta_R=2$.

\begin{figure*}[bth]
\centering
\includegraphics[width=0.75\textwidth]{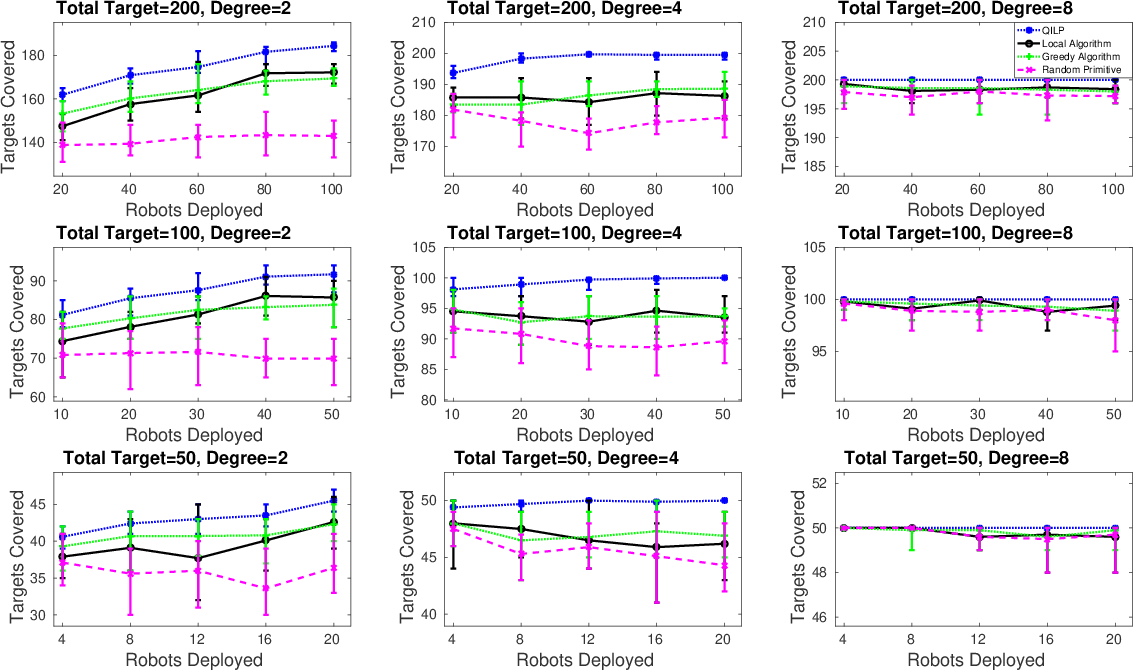}
\caption{Comparative results of QMILP, local algorithm with $h$=2, greedy algorithm, and randomly choosing a motion primitive. The plots show the mean and min/max number of targets covered from 100 trials.\label{fig:cmp_plot}}
\end{figure*}

When the number of targets are 50 and 100 with degree 4 (Figure~\ref{fig:cmp_plot}), the performance of the local algorithm does not improve as the number of robots deployed increases, which may seem counterintuitive. We conjecture that the reason behind this is the locality of the proposed algorithm. Even though more robots are used to track the same number of targets, the average degree of the target remains the same. Consequently, the communication graph for the robots becomes sparser. Since $h$ is fixed for all cases, this implies that each robot in layered graph reaches a smaller subset of the total graph, leading to even more sub-optimal performance. One avenue of future work is to analyze this in more depth.


\subsection{Multi-robot Multi-target Tracking Simulation}

The proposed local algorithm was implemented in Gazebo (Figure~\ref{fig:gazebo}). Five mobile robots were deployed to track thirty targets (a subset of which were mobile) with a FoV of 3$m$ on the ground plane. Two motion primitives were used per robot: (1) remain in place and (2) move a random distance of up to $1m$ with a random heading between $-30^{\circ}$ and $30^{\circ}$. 

\begin{figure}[thpb]
\centering
\includegraphics[width=0.65\columnwidth]{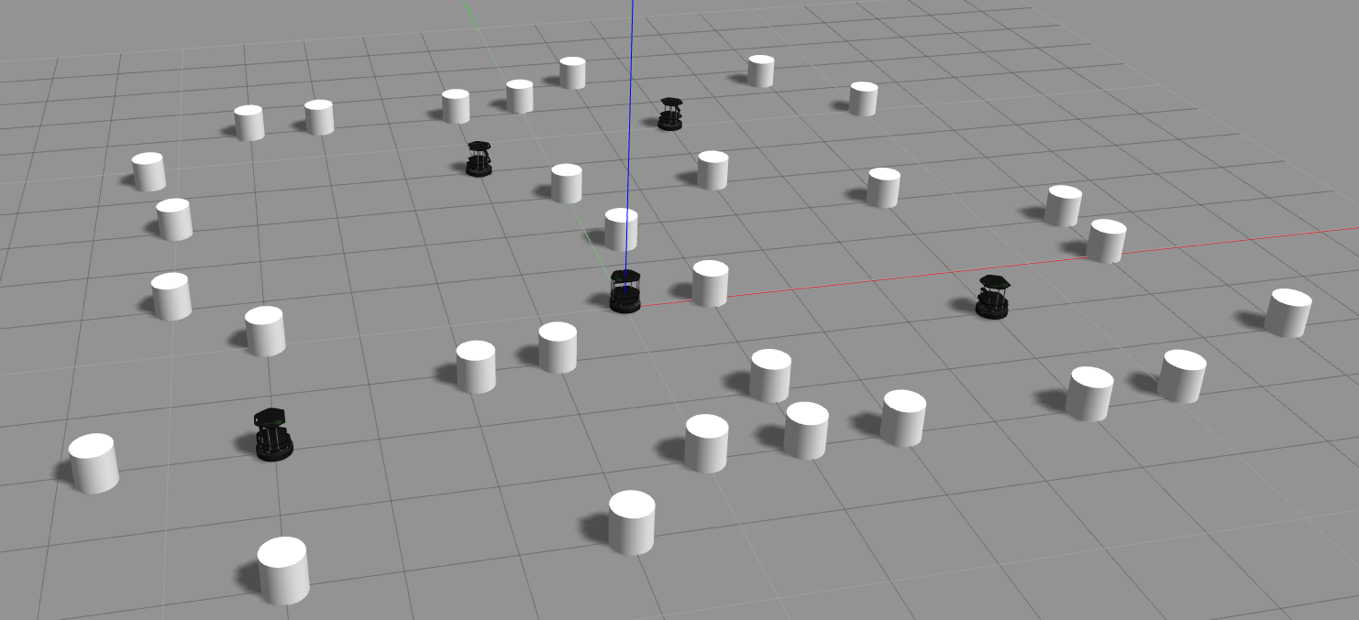}
\caption{Gazebo simulator showing five robots tracking thirty stationary and moving targets (please refer to the attached video).} 
\label{fig:gazebo}
\end{figure}

At each time step, the local algorithm chooses motion primitives to maximize the number of targets tracked (\Bottleneck version). We compared this with the greedy algorithm.
Figure~\ref{fig:trajectory} shows the trajectories of robots and targets obtained from the simulation. Figure~\ref{fig:total_num} compares the number of targets tracked by the local and greedy algorithms for a specific instance. Both algorithms have a sub-optimal performance guarantee, with the greedy algorithm having a better worst-case guarantee than the local one. However, in practice, both strategies perform comparably.

\begin{figure}[htb]
\centering
\subfigure[Trajectory information obtained by the local algorithm.]{\includegraphics[width=0.43\columnwidth]{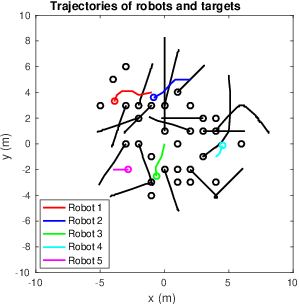}}
\subfigure[Trajectory information obtained by the greedy algorithm.]{\includegraphics[width=0.43\columnwidth]{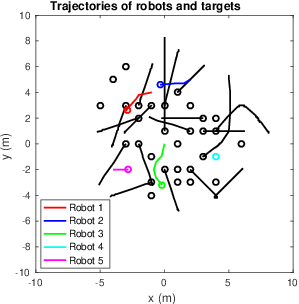}}
\caption{Plot of trajectories of robots and targets applying both local and greedy algorithms to the simulation given in Figure~\ref{fig:gazebo}. Black lines represent trajectories of thirty targets. $\circ$ denotes the end position of trajectories. Both algorithms were performed for 40 seconds on the same target trajectories.  \label{fig:trajectory}}
\end{figure}

\begin{figure}[thpb]
\centering
\includegraphics[scale=0.41]{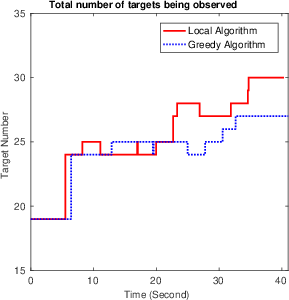}
\caption{Comparison in the total number of targets being observed by any robots between local and greedy algorithms.}
\label{fig:total_num}
\end{figure}





\section{Conclusion} \label{sec:conc}

We present a new approach for multi-robot multi-target assignment. Our work is motivated by scenarios where the robots would like to reduce their communication while still maintaining some guarantees of tracking performance. We used the local communication framework employed by Floreen et al.~\cite{floreen2011local} to leverage an algorithm that can trade-off optimality with communication complexity. We empirically evaluated this algorithm and compared it with the baseline greedy strategy. Our immediate future work is to expand the local algorithm to solve the \WinnerTakesAll version of SATA. Another extension would be to include \emph{search primitives} for exploration, in order to handle situations where some targets fall outside the FoV of all robots.




\bibliographystyle{IEEEtran}
\bibliography{IEEEabrv,yoon_refs}

\appendix
\input{appendix}

\end{document}

%% file: appendix.tex
\subsection{Proof of Lemma~\ref{lemma:greedy}}

Let $w(\textbf{t}_j)\triangleq \max\{c_{i,m}^j\big|x_m^i=1,\ \forall i, m\}$. Therefore, the sum of quality of tracking over all targets is:
\begin{equation}
\begin{split}
\sum_{\textbf{t}_j\in T} w(\textbf{t}_j)&=\sum_{\textbf{t}_j\in T}\max\{c_{i,m}^j\big|x_m^i=1,\ \forall i, m\} \\
&=\sum_{\textbf{t}_j\in T}\Big(\sum_{\textbf{r}_i\in R}\max\{\sum_{\textbf{p}_m^i\in P^i}c_{i,m}^j x_m^i\}\Big) \\
&=\sum_{\textbf{t}_j\in T}\Big(\sum_{\textbf{r}_i\in R}y_{i}^{j}\Big(\sum_{\textbf{p}_m^i\in P^i}c_{i,m}^jx_{m}^{i}\Big)\Big).
\label{eqn:sum_w_function}
\end{split}
\end{equation}

Equation~\ref{eqn:sum_w_function} is obtained by taking into account the conditional term of the first equation explicitly.
The last equation follows from the property that $y_i^j$ chooses the maximum value of $\sum_{\textbf{p}_m^i\in P^i}c_{i,m}^jx_{m}^{i}$ among all robots, which is shown in lines 10-13 of Algorithm~\ref{alg:greedy}. Therefore, the last equation is equal to the inner term of Equation~\ref{eqn:objective}.